\documentclass[letterpaper, 10 pt, conference]{ieeeconf}  % Comment this line out
\IEEEoverridecommandlockouts                              % This command is only
\usepackage{comment}
\usepackage{times}

\usepackage[linesnumbered,ruled,vlined]{algorithm2e}%[ruled,vlined]{
\usepackage{algpseudocode}
\usepackage{textcomp}

\usepackage{amsthm}
\usepackage[
bookmarksopen,
bookmarksdepth=2,
breaklinks=true
]{hyperref} 

\newtheorem{theorem}{Theorem}

\usepackage{amsmath}
\usepackage{multirow}
\usepackage{makecell,rotating,multirow,diagbox}
\usepackage{multicol}
\usepackage{subfigure}
\usepackage{float}
\usepackage{graphicx}
\usepackage{lineno}
\usepackage{cite}

\usepackage{makeidx}
\usepackage{amssymb}
\usepackage{epstopdf}
\usepackage{times}
\usepackage{helvet}
\usepackage{courier}
\usepackage{booktabs}

\usepackage{color}
\usepackage{mathrsfs}
\usepackage{amsfonts}
\usepackage{booktabs}
\usepackage[font=small,skip=3pt]{caption}

\title{\LARGE \bf
Online Multi-Target Tracking for Maneuvering Vehicles in Dynamic Road Context
}  

\author{Zehui Meng$^{1}$, Qi Heng Ho$^{1}$, Zefan Huang$^{1}$, Hongliang Guo$^{1}$, Marcelo H. Ang Jr$^{2}$ and Daniela Rus$^{3}$% <-this % stops a space  
\thanks{$^{1}$These authors are with Singapore-MIT Alliance for Research and Technology,
        Singapore
        {\tt\small \{zehui, qiheng, zefan, hongliang\}@smart.mit.edu}}%
\thanks{$^{2}$Marcelo H. Ang Jr is with Department of Mechanical Engineering, National University of Singapore, Singapore
        {\tt\small mpeangh@nus.edu.sg}}%
\thanks{$^{3}$Daniela Rus is with Computer Science \& Artificial Intelligence Laboratory, Massachusetts Institute of Technology, Cambridge, USA
        {\tt\small rus@mit.edu}}%
}

\begin{document}

\maketitle
\thispagestyle{empty}
\pagestyle{empty}
\setlength{\textfloatsep}{0pt plus 3pt minus 3pt}
\setlength{\belowcaptionskip}{0pt}
\setlength\floatsep{0.0\baselineskip plus 3pt minus 3pt}
%%%%%%%%%%%%%%%%%%%%%%%%%%%%%%%%%%%%%%%%%%%%%%%%%%%%%%%%%%%%%%%%%%%%%%%%%%%%%%%%
\begin{abstract}
Target detection and tracking provides crucial information for motion planning and decision making in autonomous driving. This paper proposes an online multi-object tracking (MOT) framework with tracking-by-detection for maneuvering vehicles under motion uncertainty in dynamic road context. We employ a point cloud based vehicle detector to provide real-time 3D bounding boxes of detected vehicles and conduct the online bipartite optimization of the maneuver-orientated data association between the detections and the targets. Kalman Filter (KF) is adopted as the backbone for multi-object tracking. In order to entertain the maneuvering uncertainty, we leverage the interacting multiple model (IMM) approach to obtain the \textit{a-posterior} residual as the cost for each association hypothesis, which is calculated with the hybrid model posterior (after mode-switch). Road context is integrated to conduct adjustments of the time varying transition probability matrix (TPM) of the IMM to regulate the maneuvers according to road segments and traffic sign/signals, with which the data association is performed in a unified spatial-temporal fashion. Experiments show our framework is able to effectively track multiple vehicles with maneuvers subject to dynamic road context and localization drift. 

\end{abstract}
%The deterministic data association provides analytical evaluation of the spatial and/or temporal correlation shift (changes of structural and temporal constraints despite the process noise and measurement noise), while allowing road context to be integrated as prior distributions so as to regulate the predicted posterior.

%%%%%%%%%%%%%%%%%%%%%%%%%%%%%%%%%%%%%%%%%%%%%%%%%%%%%%%%%%%%%%%%%%%%%%%%%%%%%%%%
\section{Introduction}
In autonomous driving, knowing the motions of road-sharing traffic participants is critical to the decision making and motion planning of the ego-car. This frequently requires fast, reliable object detection and tracking modules that provide real-time pose estimations and updates.

The problem of single target tracking has been well addressed by filter-based methods \cite{KF,Isard1998,978374}. In presence of unknown number of targets and noisy sensor data, the more challenging multiple object tracking (MOT) task triggers the difficulty of data association induced by multiple hypotheses subjected to ambiguity. Combinations of the pairing-up are evaluated with designated metrics to sort out the optimal distribution/assignment. Approaches like Joint Probabilistic Data Association (JPDA) \cite{JPDA}, Multiple Hypothesis Tracking (MHT) \cite{MHT}, and Finite Set Statistics (FISST) \cite{8289374} are proven effective to solving the combinatorial assignment. As the problems size grows, Sequential Importance Resampling (SIR) based Filtering methods \cite{Isard1998,978374,MCMC} become more preferred for efficient hypotheses selection and propagation. Recently, as computer vision techniques mature, more and more works are tackling the problem via tracking-by-detection \cite{tbd,tbd1,tbd2,5459278,5674059}, which couples the detection with tracking to formulate the discrete optimization problem of data association. For the sake of handling maneuvering objects (e.g. vehicle turning, stopping, overtaking, etc.), the Interacting Multiple Model (IMM) \cite{IMM} paradigm demonstrates superior ability to identify the mode switch. 
%Meanwhile, learned affinity features also eased the problem of selecting proper data association metrics \cite{6520846,DBLP:journals/corr/Choi15,7780524,DBLP:journals/corr/WangF16b,7410891,5708151}.

In this paper, we focus on the MOT problem of vehicle tracking under maneuvering uncertainty and address the issue using an online approach. We propose an MOT framework for online tracking using tracking-by-detection while adopting a model posterior based data association that entertains the maneuver uncertainties regulated by road context. Our major contributions are summarized as follow: \textbf{(1).} We provide an online tracking-by-detection pipeline for MOT tracking under uncertainty for maneuvering vehicles. \textbf{(2).} We propose a deterministic data association metric for maneuvering targets - exploiting the \textit{a-posterior} residual calculated by the model posterior using interacting multiple model (IMM). \textbf{(3).} We solve the spatial and temporal data association in a unified fashion via integrating the road context as time-varying transition probability matrix (TPM). \textbf{(4).} We demonstrate experimentally the superiority of our framework against baselines methods tracking multiple maneuvering vehicles regulated by dynamic road context, with stable tracking achieved even under abrupt localization drift. 

The rest of the paper is organized as follow. \textbf{Section II} briefs the background and related works. \textbf{Section III} presents our framework. \textbf{Section IV} demonstrates the experimental results and analysis. \textbf{Section V} gives the conclusion.      

\begin{figure*}[t!]
\begin{minipage}{0.65\linewidth}
\centering
    \includegraphics[width=\linewidth,height=0.6\linewidth]{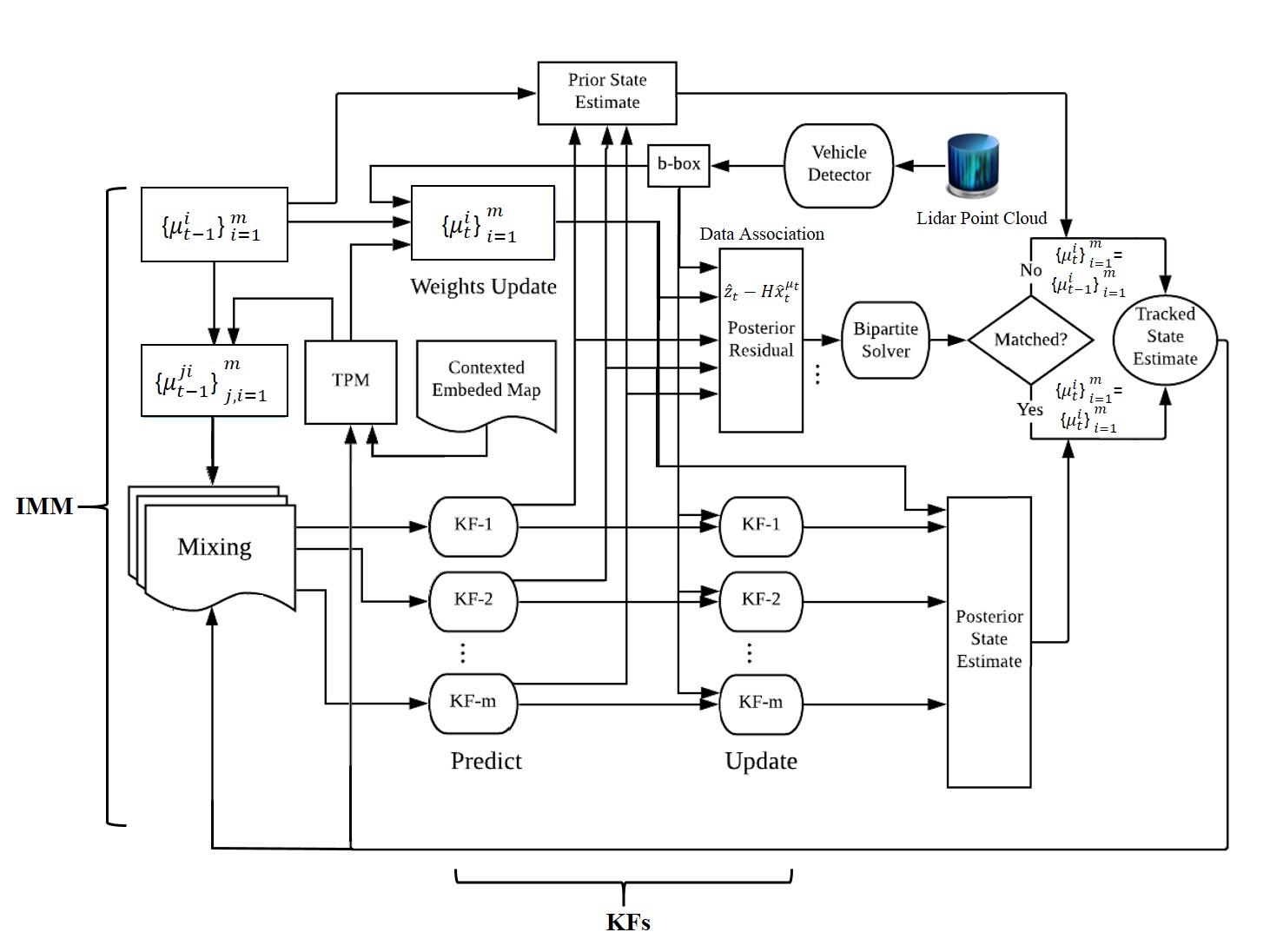}
    %\captionof{figure}{The detection and tracking framework overview.}
    \vspace{-1em}
    \label{fig:overview}
\end{minipage}
\begin{minipage}{0.35\linewidth}
\centering
    \includegraphics[width=.9\linewidth,height=\linewidth]{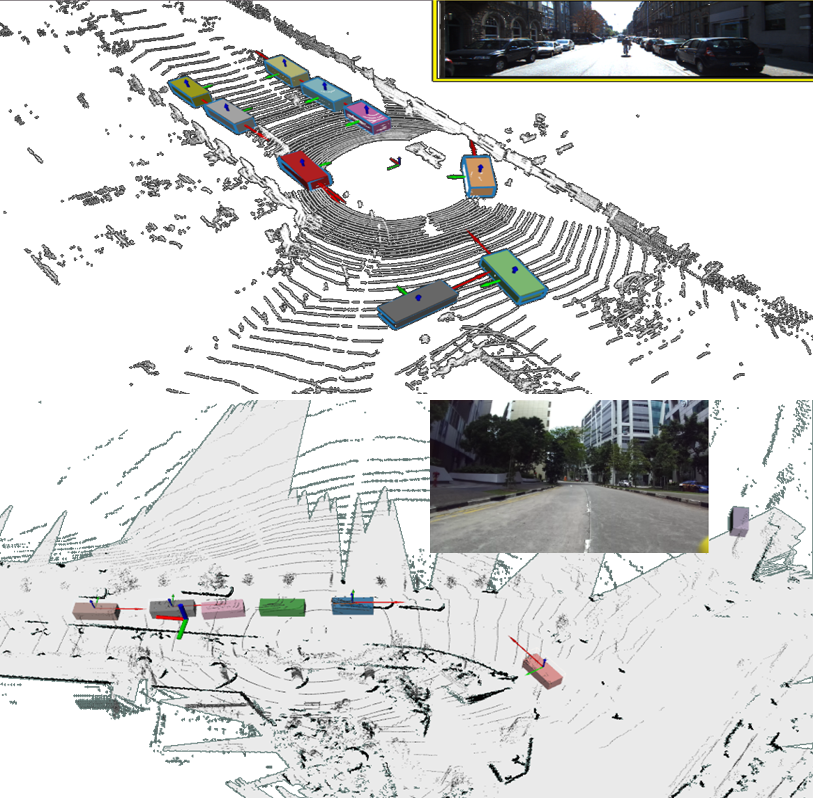}
    %\captionof{figure}{The detection and tracking instances in KITTI (top) and Singapore roads (bottom).}
    \vspace{-1em}
    \label{fig:instance}
\end{minipage}
\caption{Our detection and tracking framework overview and instances in KITTI dataset (top right) and Singapore roads (bottom right).}
\label{fig:overview}
\vspace{-1em}
\end{figure*}

\section{Background \& Related Works}
\begin{comment}
\subsection{Single-target tracking}
For single target tracking (SOT), two derivatives of Bayesian filters are widely used. 
The first approach - Kalman filter \cite{KF}, assuming target states as Gaussian distributions, adopts a control predict step and a measurement update step, which outputs optimal state estimations for linear motion and sensor model. Several suboptimal extensions such as Extended Kalman Filter (EKF) \cite{KF} and Unscented Kalman Filter (UKF) \cite{KF} further expand their capabilities to non-linear models and additional constraints. The second approach - Particle filter, usually accompanied with sequential Monte Carlo method (SMC), approximates arbitrary probability density functions with discrete particles \cite{Isard1998,978374}, yielding a fast tracking method while capable with both linear and nonlinear systems.
\end{comment}

\subsection{Multi-target tracking}
Single target tracking (SOT) can be well addressed by Bayesian filters such as Kalman filter \cite{KF} and particle filter \cite{Isard1998,978374} as well as their variants; while multi-object tracking (MOT) brings the problem of linking the measurements to the targets,
known as data association. Approaches to MOT can be categorized by how they handle the data association.

\subsubsection{Deterministic vs Probabilistic}

Data association can be handled either deterministically or probabilistically. 

Deterministic association commonly assigns each measurement a particular target or claim mismatch. Classical works employ affinity evaluation like Nearest Neighbor \cite{NN} or color histograms \cite{MAP}, while modern works are more interested in learned features in conjunction with motion models \cite{DBLP:journals/corr/Choi15,6751397,6909555,5540148,10.1007/978-3-540-88690-7_48,10.1007/978-3-642-37431-9_8}. Towards optimal assignment, Multiple Hypothesis Tracking (MHT) \cite{MHT} offers a breadth-first-search style optimization by maintaining a tree of all possible association hypotheses and propagating the track likelihoods. Tree pruning and k-best assignment techniques \cite{481539} can be employed to alleviate the combinatorial explosion. 

Alternatively, probabilistic association calculates a probabilistic distribution of all measurements over the tracks. The probabilistic version of MHT (PMHT) \cite{PMHT} adopts a maximum-likelihood estimation (derived from expectation-maximization) of the assignments to propagate the hypotheses. Joint Probabilistic Data Association filter (JPDAF)~\cite{JPDA} offers Bayesian approximation of the Gaussian posteriors of all possible assignments to each track, weighted by their respective hypothesis likelihood \cite{Isard1998}. For large search spaces, Sequential Importance Resampling (SIR) based data association like the Markov Chain Monte Carlo (MCMC) \cite{MCMC,1512059} and the Rao-Blackwellized Monte Carlo (RBMC) 
\cite{Rao,Rao2} provide more efficient sampling of the hypotheses with asymptotic convergence guarantee.

\subsubsection{Batch Optimization vs Online Association} 
Data association can be processed either online or in batch.

Batch approaches consider the optimization of the entire sequence as a global optimal solution, which often model the problem as a directed graph with binary constraints and solve for the min-cost network flow \cite{MAP,5995604,6520846,Li09learningto,5708151,5540148}. On the other hand, online approaches deal with the tracking by formulate the matching between consecutive frames into a series of bipartite linear assignment problem \cite{10.1007/978-3-540-88690-7_48,10.1007/978-3-642-37431-9_8,6909555,5674059,MCMC,1512059}. The data association is usually optimized greedily in each (local) sub-component of the whole sequence to achieve real-time efficiency. 

\subsubsection{Maneuvering Target Tracking}
For dealing with maneuvering targets, the Interacting Multiple Model (IMM) filters \cite{IMM} maintains multiple weighted models which are processed in parallel through Bayesian filtering to form a hybrid filter. The models are re-weighted with new measurements iteratively, with which the overall estimate is calculated as a Gaussian mixture. The IMM is proven to be the most cost-effective state estimator of highly maneuverable objects, and frequently used in conjunction with JPDA \cite{Kampker2018TowardsMD}. 

\subsection{Tracking-by-Detection Pros and Cons} 
Tracking-by-detection is nowadays a favored paradigm for coupled detection and tracking \cite{tbd,tbd1,tbd2,5459278,5674059}. Differring from generative filtering approaches trying to estimate continuous trajectories with observations, it converts the problem to a discrete linear assignment. Usually a target detector is selected to perform frame-by-frame detections (marking out target locations and orientations) which are correlated to form individual tracklets of the targets, as per data-association metric employed \cite{6520846,DBLP:journals/corr/Choi15,7780524,DBLP:journals/corr/WangF16b,7410891,5708151,8460951}. These approaches have in-principle simplified the hypotheses sampling and filtering process with the use of detection cues (e.g., appearance features, etc.), saving computation effort from the combinatorial complexity of raw measurements. 

However, many of these works are ad-hoc to specific measurement source or sensory modalities, especially those leveraging learned features as affinity functions \cite{5540148,6909555,10.1007/978-3-642-37431-9_8,Li09learningto, 10.1007/978-3-540-88690-7_48}, ending up with particular data association regimes that may not be universally applicable across platforms or detectors. In absence of such discriminative cues, highly maneuvering objects can become problematic to handle. Relying on the deterministic structural constraints (e.g. bounding box IoUs) and fixed model prior based prediction-measurement association metric, the maneuver causing mode-switch is hardly entertained and likely leads to ambiguous association. 

With the above said, we consider the effect of model shift for discriminating the affinity of deterministic online detection-target associations, which could potentially benefit the real-time tracking under maneuver uncertainty.

\section{Methodology}
%\section{MOT under Road Context Regulated Maneuvering}
In this section, we present our proposed MOT framework which extends general tracking-by-detection paradigm's capability to entertain maneuvering targets and incorporate the road context as constraints to regulate the data association. An overview of the framework is shown in~\textbf{Fig.\ref{fig:overview}}.

\begin{comment}
\begin{figure}[htb]
\centering
\subfigure[]{
\label{fig.mean-var}
\includegraphics[height=3.22cm,width=3.75cm]{Mean_std_estimation.pdf}}
\subfigure[]{
\label{fig.uncertainty}
\includegraphics[height=3.22cm,width=3.75cm]{Entropy_obstacle.pdf}}
\caption{(a) The mean and standard deviation of the distance between the estimated posterior mean and the true location of the obstacle. Blue color shows standard deviation for $200$ independent simulation runs; (b) The uncertainty of the obstacle's location VS. time steps.}
\label{fig:entropy}
\end{figure}
\end{comment}

\subsection{Target Detector}
We target the point cloud from 3D Lidar as our main sensor measurement. Unlike the camera images, the point cloud is sparser and less informative for discriminative appearance features to take effect. Instead of counting on the appearance affinity, we utilize only the general geometric detection results for data association. We employ a CNN-based 3D object detector developed on top of the Sparsely Embedded Convolutional Detection (SECOND \cite{second}). The network takes in raw 3D point cloud ($360^o$) and outputs the detection results in terms of objectness ($s$) and 3D bounding boxes (dimension $(l,w,h)$, position $(x,y,z)$, and yaw angle $\theta$). We implement auxiliary variational encoders on top of the voxel-feature extraction layers, which serves to provide generative voxel features for augmenting sparse point cloud \cite{SPPAM}. The model is trained on the Kitti Dataset \cite{KITTI} and achieves a moderate performance with 20 Hz frame rate.   
\subsection{Kalman Filter based MOT Tracker}
Kalman Filter is employed as the backbone for our MOT tracker which is in nature a collection of multiple SOT KF-trackers maintained with online data association and a lifespan control scheme. For each frame, assume we have varying numbers of detections and tracks for linear assignment.   

\subsubsection{KF Model}
For state estimation, typical Kalman filter assumes linear motion ($F$, control input omitted) and sensor ($H$) models with Gaussian noises ($Q,R$):
\small
\begin{align} 
    & x_t = Fx_{t-1} + \omega, &\omega\sim\mathcal{N}(0,Q) \\
    & z_t = Hx_t+ \nu, &\nu\sim\mathcal{N}(0,R)
\end{align}
\normalsize
where $x_t$ and $z_t$ are the state and measurement at frame $t$. 
In our context, the state $x_t = (l,w,h,x,y,z,\theta,\dot{x},\dot{y},\dot{z},\dot{\theta})$, the measurement $z_t = (l,w,h,x,y,z,\theta)$.

\subsubsection{Track Lifespan Management} 
The lifespan management controls the registration, propagation, and termination of individual SOT KF-trackers. A dedicated KF-tracklet is initialized in case of each unmatched detection (detection without a match in the tracks), allowing it to be propagated forward. To alleviate the influence of missing detection and false positives, the tracklet is not registered to become valid tracks until consecutive matching with the next-few-frame detections is secured (i.e. lifespan must be greater than a few number of frames to survive). Likewise, unmatched track (track without any matching detections) is not immediately terminated but still allowed to carry on until constant mismatch is seen throughout the next few frames. The lifespan scheme enables filtering of false or missing detections.    

\subsubsection{Prediction}
For newly transferred tracklets and existing tracks, the states are propagated with the KF prediction step:
\small
\begin{flalign} 
    & \hat{x_t} = F\tilde{x_t}_{-1} \label{KF1}\\
    & \hat{P_t} = F\tilde{P_t}_{-1}F^T+ Q 
\end{flalign}
\normalsize
where $\hat{x_t}, \hat{P_t}$ are the predicted state mean and covariance.

\subsubsection{Update}
Only for matched detection-target pairs, the corresponding track state is updated with the KF update step:
\small
\begin{flalign} 
    & K_t = \hat{P_t}H^T(H\hat{P_t}H^T+R)^{-1} \\
    & \tilde{x_t} = \hat{x_t} + K_t(\hat{z_t}-H\hat{x_t}) \\
    & \tilde{P_t} = (I-K_tH)\hat{P_t} 
\end{flalign}
\normalsize
where $K_t$ is the optimal Kalman gain and $\hat{z_t}$ is the measurement observed at frame $t$.

\subsection{Data Association}
The data association in our proposed framework considers the potential maneuvering behavior of the targets, adopting the \textit{a-posterior} residual calculated using the model posterior.

\subsubsection{Error Minimization}
In the common practice of IoU or other affinity score based data association, the deterministic mean values of the predicted state $\hat{x_t}$ and the observed measurement $\hat{z_t}$ are adopted to equivalently minimize the difference between $\hat{z_t}$ and $H\hat{x_t}$, which are independent of $P,Q,R$. Intuitively, this seems to only minimize the difference between the observed value and the expected value, known as the residual (a.k.a. innovation). Yet we show in the following part it is equivalent to minimizing the {\it a-priori} and {\it a-posteriori} distribution error given specific models:
\begin{theorem}\label{theo:Theorem1}
Given the KF model with known $F,H,Q,R$, minimizing $(\hat{z_t}-H\hat{x_t})$ is equivalent to minimizing the KL-divergence of {\it a-priori} $p(x_t|z_{1:t-1})\sim\mathcal{N}(x_t;\hat{x_t},\hat{P_t})$ and {\it a-posteriori} $p(x_t|z_{1:t})\sim\mathcal{N}(x_t;\tilde{x_t}, \tilde{P_t})$.
\end{theorem}

\begin{proof}

\begin{comment}
As the Kalman filter serves as an optimal estimator of the true state $x_t$, the estimated state (both mean $\tilde{x_t}$ and variance $\tilde{P_t}$) is dependent on the Kalman gain $K_t$ minimizing the posterior mean square error $E\Big[||x_t-\tilde{x_t}||^2\Big]$, which is equivalent to minimizing the trace of the covariance $\tilde{P_t}$. From \textbf{Eq.(3) - Eq.(7)} we can see that $K_t,\hat{P_t},\tilde{P_t}$ are independent of the measurement mean value $\hat{z_t}$ but $F,H,Q,R$ only. 
\end{comment}
\begin{comment}
Whereas $\hat{x_t}$ and $\hat{z_t}$ are by definition solely characterized by $F$ and $H$, leaving out the uncertainty terms $P,Q,R$. This is in fact not leveraging the Kalman filter's optimality since the assigned detection-target pair is not conditioned on the uncertainties. 
\end{comment}

According to Bayes' Rule, the resultant state estimate distribution given the current observation at frame $t$ is:
\small
\begin{flalign} 
    & p(x_t|z_{1:t}) = \frac{p(x_t|z_{1:t-1})p(z_t|x_t)}{p(z_t|z_{1:t-1})}
\end{flalign}
\normalsize
where $p(x_t|z_{1:t-1})$, $p(x_t|z_{1:t})$, $p(z_t|x_t)$ correspond to $\mathcal{N}(x_t;\hat{x_t},\hat{P_t})$, $\mathcal{N}(x_t;\tilde{x_t}, \tilde{P_t})$, $\mathcal{N}(z_t;Hx_t, R)$, respectively, and $p(z_t|z_{1:t-1})$ the normalizing factor.
%$p(x_t,z_t|z_{1:t-1})$ to $\mathcal{N}(H\hat{x_t}, H\hat{P_t}H^T+R)$.
\begin{comment}
To evaluate how good the match between the above two distributions finds their correlations, the mutual information can be calculated:
\begin{flalign}
     & MI(P(x_t|z_{1:t-1}),P(z_t|x_t)) \nonumber\\ 
     & \propto D_{kl}(P(x_t|z_{1:t})||P(x_t|z_{1:t-1})P(z_t|x_t)) \nonumber\\
     &  = \frac{1}{2}log\frac{\det{\hat{P_t}}\det{R}}{\det{P_t}}
\end{flalign}
where $D_{kl}$ stands for the KL-Divergence.
\end{comment}

We aim to minimize the distribution error between the {\it a-priori}   $p(x_t|z_{1:t-1})\sim\mathcal{N}(\hat{x_t},\hat{P_t})$ and the {\it a-posteriori} ~$p(x_t|z_{1:t})\sim\mathcal{N}(\tilde{x_t}, \tilde{P_t})$, both in $\mathbb{R}^n$. A straightforward evaluation metric is the KL-Divergence:
\small
\begin{flalign} 
      D_{KL} & (p(x_t|z_{1:t-1})||p(x_t|z_{1:t})) \nonumber\\
      = & E_{p(x_t|z_{1:t-1})} \Big[\log p(x_t|z_{1:t-1}) - \log p(x_t|z_{1:t})\Big] \nonumber\\
      = & \frac{1}{2}\Big(\log\frac{\det{\tilde{P_t}}}{\det{\hat{P_t}}}-n+ tr(\tilde{P_t}^{-1}\hat{P_t}) \nonumber\\
          &  + (\tilde{x_t}-\hat{x_t})^T \tilde{P_t}^{-1}(\tilde{x_t}-\hat{x_t}) \Big)
\end{flalign}
\normalsize
where $\tilde{x_t}-\hat{x_t}$ is exactly $K_t(\hat{z_t}-H\hat{x_t})$ as in \textbf{Eq.(6)}. 

Recall that for a given model, from \textbf{Eq.(3) - Eq.(7)} we can see that $K_t,\hat{P_t},\tilde{P_t}$ are independent of the measurement mean value $\hat{z_t}$ but $F,H,Q,R$ only, thus can be calculated offline as constants. Thus minimizing the $D_{KL}$ is equivalent to minimizing the term $||\hat{z_t}-H\hat{x_t}||_2$. 
\end{proof}

\begin{comment}
the third term $NMI(\Sigma_1,\Sigma_2,\Sigma_{12})\in[0,1]$ represent the normalized mutual information of two random distributions $P1,P2$ with variances $\Sigma_1,\Sigma_1$ and joint covariance $\Sigma_{12}$.    
\end{comment}

\subsubsection{Handling the Maneuvers}
The state transitions can change patterns due to interruptions and maneuvers. To account for that, the linear model of Kalman filter generally takes an additional control input term $Bu$, with which the model description in \textbf{Eq.(1)} can be revised as:
\small
\begin{align} 
    & x_t = Fx_{t-1} + Bu_t +\omega, & \omega\sim\mathcal{N}(0,Q) 
\end{align}
\normalsize
where $u_t$ is the input variable, $B$ is the transformation matrix.

\begin{comment}
However, external observers such as our tracker usually have no access to the control input of the target, i.e. a deterministic $u_t$ is not known to us. In such case, we assume a normally distributed input:
\begin{align} 
    & u_t\sim\mathcal{N}(\hat{u_t},\Sigma)
\end{align}
For the prediction step, the \textbf{Eq.(3), Eq.(4)} ought to be revised as:
\begin{flalign} 
    & \hat{x_t}' = F\tilde{x_t}_{-1} + B\hat{u_t} \\
    & \hat{P_t}' = F\tilde{P_t}_{-1}F^T + B\Sigma B^T+ Q 
\end{flalign}
where $B\Sigma B^T$ is in $\mathbb{R}^{n\times n}$. 
\end{comment}

\begin{comment}
Since we stick to \textbf{Eq.(4)} as our prediction model instead of \textbf{Eq.(13)}, the uncertainty $Bu_tB^T$ shall be introduced additionally to rectify the matching conditioned on the control input. We instead impose this part of uncertainty on the sensor model:
\begin{align}
    & P(z_t|x_t)\sim\mathcal{N}(Hx_t,~R+B\Sigma B^T) 
\end{align}
and correspondingly, $K_t$ and $\tilde{P_t}$ shall be revised as:
\begin{flalign}
    & K_t' = \hat{P_t}H^T(H\hat{P_t}H^T+R+B\Sigma B^T)^{-1} \\
    & \tilde{P_t}' = (I-K_t'H)\hat{P_t}
\end{flalign}

Thus the rectified KL-Divergence is now sought between $\hat{P_t}'$ and the $\tilde{P_t}$:
\begin{flalign}
    & \Bar{D}_{KL}(\hat{P_t}',\tilde{P_t}) 
\end{flalign}
where $\hat{P_t}'$ and $\tilde{P_t}$ are calculated by \textbf{Eq.(13)} and \textbf{Eq.(7)} respectively.
\end{comment}

The control $u$ might be something like acceleration, in which case a constant acceleration model $B$ can be assumed. If linear motions are assumed throughout the tracking, the above can be approximated with a new motion model $G$:
\small
\begin{flalign} 
    & \hat{x_t}^G = G\tilde{x_t}_{-1} \\
    & \hat{P_t}^G = G\tilde{P_t}_{-1}G^T + Q 
\end{flalign}
\normalsize
with which the $K_t^G, \tilde{x_t}^G, \tilde{P_t}^G$ can be revised accordingly.

In circumstances of dynamically maneuvering models, it is no longer optimal to minimize the term $||\hat{z_t}-H\hat{x_t}||_2$ with previous motion model $F$ for the best match, but to use $G$ instead to get the minimum $||\hat{z_t}-H\hat{x_t}^G||_2$. 

\begin{comment}
As we have no clues what the new model $G$ looks like, the set of all possible $G$ matrices need to be searched to get the minimum $\hat{z_t}-H\hat{x_t}^G$. The closed form solution is intractable as there exist unknown number of models or model mixtures to make $\hat{z_t}-HG\tilde{x_t}_{-1}=0$.
\end{comment}

On the other hand, in many situations one can expect the targets to operate with ruled maneuvers, such as the vehicles on road following the traffic rules. Such context information can be adopted as constraints to regulate the maneuvers. Intuitively, the road context provides instructions to limit the vehicle motion models shift (e.g. the car can hardly move sideways in narrow straight lanes but can turn abruptly at intersections and roundabouts). Under dynamic road context, the motion model ought to be frequently adjusted to keep track of the targets. 

\begin{comment}
The motion change can actually be interpreted from the covariance matrices associated with the states, with the entries indicating the correlations among the variables. In Kalman filter for instance, the state transition function $F/G$ and the measurement function $H$, together with the corresponding noises $Q$ and $R$, are jointly characterizing the state covariance $P$. The entries in $P$ suggests how the variables are co-varying with each other. In cases of mutually independent variables, $P$ will take up the shape of a diagonal matrix. As such, we also propose to incorporate the covariance terms into the data association, to force the linear assignment to be conditioned on the uncertainties.
\end{comment}

\subsubsection{Exploiting the Model Shift}
Regarding the above issues of maneuvering target with dynamic road context, a possible approach is to impose the road context constraints on the motion matrices $G$, characterized by spatial and temporal relevance at different road segments. The linear assignment become the optimization of the bipartite graph: 
\small
\begin{flalign} 
     & \underset{}{\textbf{minimize}} ~~ \sum_p^M \sum_q^N ||\hat{z}_t^p-(H{\hat{x}_t^q})^G||_2 \nonumber\\
     & \textbf{subject to} ~~ \textit{context-aware constraints}
\end{flalign}
\normalsize
where $M,N$ are the number of detection and tracks. 

However, doing so requires plenty effort imposing carefully hand-crafted correlations of each best fitting $G$ matrix exactly complying with the constraints, and sometimes impossible due to the non-linearity in dynamics (e.g. deriving Jacobians for EKF). A quick fix is to use primitive linear models to approximate $G$. We can still employ common motion models (e.g. constant velocity/acceleration) as primitives to estimate the resultant motion. But when and how the models switch to one another shall be heuristically encoded following the road context. This naturally leads to the interacting multiple model (IMM) filter.

The interacting multiple model (IMM) filter is a rescuer solution to tracking such mode-switching targets by mixing multiple models as a time-varying hybrid model to estimate the states. The IMM approximates the posterior state estimate as a Gaussian mixture of that given by multiple filters, which are weighted by the posterior mode probabilities $\{\mu_t^i\}_{i=1}^m$:
\small
\begin{flalign}
    \mu_t^i & = p(F_t=i|z_{1:t}) \nonumber\\
    & = \frac{p(z_t|F_t=i,z_{1:t-1})p(F_t=i|z_{1:t-1})}{p(z_t|z_{1:t-1})}\nonumber\\
    & = \frac{\mathcal{N}(z_t;H\hat{x_t}^i,H\hat{P_t}^i H^T+R)\sum_{j=1}^m Pr_{ji}\mu_{t-1}^j}
    {\sum_{k=1}^m \mathcal{N}(z_t;H\hat{x_t}^k,H\hat{P_t}^k H^T+R)\sum_{j=1}^m Pr_{jk}\mu_{t-1}^j}
\end{flalign}
\normalsize
where $i=1,\ldots,m$ stands for the $i^{th}$ model of the $m$ primitive motion models, and $Pr_{ji}=p(F_t=i|F_{t-1}=j)$ is the (time-invariant) mode transition probability from the $j^{th}$ to the $i^{th}$, at frame $t$.

By calculating the mixing probabilities $\{\mu_{t-1}^{ji}\}_{j,i=1}^m$ for the corresponding KF filter of the $i^{th}$ model, the individual model-specific prediction and update steps can be calculated based on the mixed state estimates $\{\bar{x}_{t-1}^i, \bar{P}_{t-1}^i\}_{i=1}^m$ :
\small
\begin{flalign}
    \mu_{t-1}^{ji} & = p(F_{t-1}=j|F_t=i, z_{1:t-1}) \nonumber\\
    & = \frac{p(F_t=i|F_{t-1}=j, z_{1:t-1})p(F_{t-1}=j|z_{1:t-1})}
    {p(F_t=i|z_{1:t-1})}\nonumber\\
    & = \frac{Pr_{ji}\mu_{t-1}^j}{\sum_{k=1}^m Pr_{ki}\mu_{t-1}^k}
\end{flalign}
\begin{flalign}
   & \bar{x}_{t-1}^i = \sum_{j=1}^m \mu_{t-1}^{ji} \tilde{x_t}_{-1}^j \\
   & \bar{P}_{t-1}^i = \sum_{j=1}^m \mu_{t-1}^{ji} \Big(\tilde{P_t}_{-1}^j + (\tilde{x_t}_{-1}^j - \bar{x}_{t-1}^i) (\tilde{x_t}_{-1}^j - \bar{x}_{t-1}^i)^T \Big)\\
   & \tilde{x_t}^i, \tilde{P_t}^i = KF(\bar{x}_{t-1}^i, \bar{P}_{t-1}^i)
\end{flalign}
\normalsize
where $KF$ follows the Kalman filter model in $\textbf{Eq.(3)-(7)}$.

The resultant overall state estimate $\bar{x_t}, \bar{P_t}$ can them be calculated as standard Gaussian mixture:
\small
\begin{flalign}
   & \bar{x_t} = \sum_{j=1}^m \mu_{t}^{i} \tilde{x_t}^i \\
   & \bar{P_t} = \sum_{j=1}^m \mu_{t}^{i} \Big(\tilde{P_t}^i + (\tilde{x_t}^i - \bar{x_t}) (\tilde{x_t}^i - \bar{x_t})^T \Big)
\end{flalign}
\normalsize

We employ the IMM as the potential model change ``indicator" to evaluate the data association error. As the IMM updates the model weights after each filtering iteration, the underlying hybrid model characterizing the state transition to associate with a certain measurement $\hat{z_t}$ can be recognized as a mixture of models with posterior weights $\mu_t=\{\mu_t^i\}_{i=1}^m$. In such case, the corresponding posterior weights (mode probabilities) can be utilized to yield a heuristic of $G$ to evaluate the bipartite  detection-track association: 
\small
\begin{flalign} 
     & \underset{}{\textbf{minimize}} ~~ \sum_p^M \sum_q^N ||\hat{z}_t^p-(H{\hat{x}_t^q})^{\mu_t}||_2 \nonumber\\
     & \textbf{subject to} ~~ \textit{context-aware constraints} %\Bar{D}_{KL}(\hat{P_t}^*||\hat{P_t}^{\mu_t}) ~ \leq ~ \epsilon 
\end{flalign}
\normalsize
where each $\hat{x_t}^{\mu_t}$ is calculated with the posterior hybrid model weights $\mu_t$ (after mode probability update): %, while $\hat{P_t}^\mu^o$ using $\mu^o$ before mode switching.
\small
\begin{flalign} 
    & \hat{x_t}^{\mu_t} = \sum_{j=1}^m \mu_{t}^{i} F_t^i \tilde{x}_{t-1}^i
\end{flalign}
\normalsize

\subsubsection{Incorporating the Road Context}

\begin{figure}[t!]
\centering
  %{\includegraphics[width=\linewidth,height=0.6\linewidth]{image/tracking_kitti_2.png}}
  %\hfill
  {\includegraphics[width=.95\linewidth,height=0.5\linewidth]{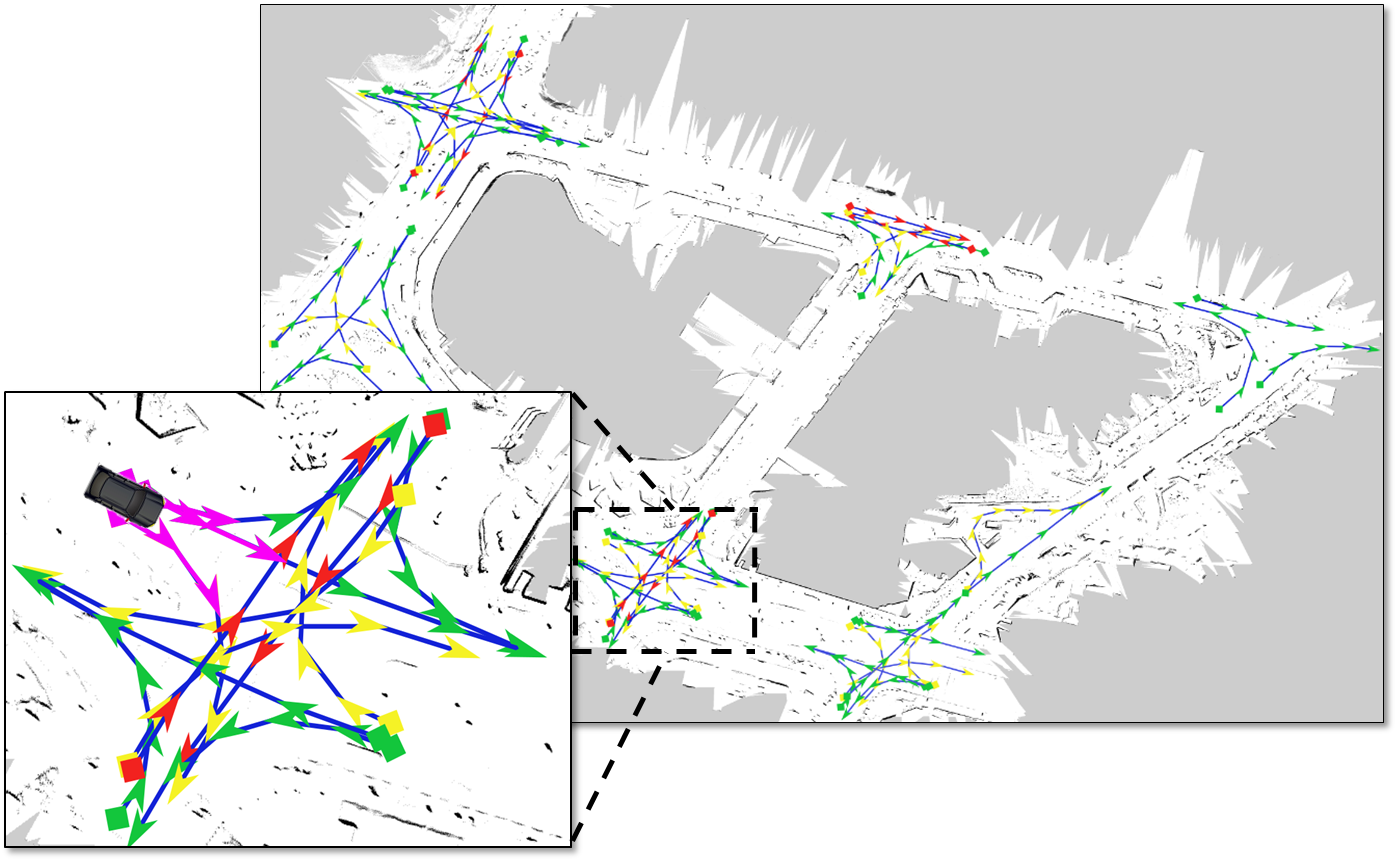}}
\caption{The road context vectors embedded in digital map. The vectors colored in blue describe basic geometric constraints following the road segments and traffic signs. The associated arrows colored in red, yellow, and green corresponds to constraints following the traffic light signal stop, slow-down/warn, and go, respectively. The vectors highlighted in magenta in zoomed view illustrates the instance of nearest neighbour context vectors activated by a vehicle.}
\label{fig:context}
\end{figure}

With such a setup, the road context can be integrated into the IMM to regulate the maneuvers. The interface is intuitively the transition probability matrix (TPM) $Pr$. We cast the transition probability $Pr_{ji}^t=p(F_t=i|F_{t-1}=j,z_{1:t})$ a time varying function characterized by the priors of the road contexts in the prescribed map - including traffic signs and different road segments. The current overall state estimate of each target $\bar{x_t}$ is transformed into the map frame, which is compared to the road context priors in the form of directed unit vectors $V_c$ with velocity toggle values $\tau \in \{0,0.5,1\}$ to indicate the status of traffic like stop, go, and slowing down situations (e.g. traffic lights can be integrated). A illustration is shown in \textbf{Fig.~\ref{fig:context}}, where the road context information is embedded into the map as directed vectors, each with a unique TPM. 

Each directed vector $V_c^i$ is prescribed with dedicated TPM $Pr^i$. The cross product of the target vehicle state and the nearby guiding context vectors activated (\textbf{Fig.~\ref{fig:context}} magenta arrows) by nearest-neighbour search  are obtained to calculate the heuristic likelihood that the vehicle is about to follow which context, the result of which is further regulated by the velocity toggles to comply with different road sections: 
\small
\begin{flalign} 
     & p_t^i = p(Pr^t = Pr^i|\bar{x_t},V_c^i) \triangleq \eta \cdot (H\bar{x_t} \times
     V_c^i) \cdot \tau
\end{flalign}
\normalsize
where $\eta$ is the normalizing factor.

Then the TPM $Pr^t$ at each target state $\bar{x_t}$ is calculated as the weighted sum of the TPMs associated with the nearby $k$ road context vectors:
\small
\begin{flalign} 
     & Pr^t = \sum_{i=1}^k p_t^i Pr^i 
\end{flalign}
\normalsize

%Correspondingly, the mode probability $\mu_t^i$ in \textbf{Eq.(14)} and the mixing prbability $\mu_{-1}^{ji}$ will need to be calculated with $Pr^t$ and $Pr^{t-1}$ respectively.  

\subsubsection{Solving the Bipartite Matching}
As mentioned, the tracker is propagated with online data association to achieve real-time performance. The bipartite problem is solved regarding adjacent frames. 
As a common practice, we first conduct a gating (with relatively relaxed spatial constraints) to filter out detections and tracks that are unlikely to have connections in their neighbourhood, which are straightaway taken care of by the lifespan manager.
For the remaining detections and tracks, the linear assignment is performed with the Hungarian Algorithm \cite{Kuhn2010TheHM}, employing our proposed association metric - the \textit{a-posterior} residual. The Hungarian Algorithm essentially carries out an combinatorial optimization of the bipartite graph composed of the detections and tracks. The resulting trackers are handled by the lifespan manager as described previously.

\section{Experiments}
\begin{comment}
\begin{figure}[t!]
\centering
  {\includegraphics[width=0.49\linewidth,height=0.25\linewidth]{image/tracking_kitti_1.png}}
  \hfill
  {\includegraphics[width=0.49\linewidth,height=0.25\linewidth]{image/tracking_sg_1.png}}
\caption{The detection and tracking instances in KITTI (left) and Singapore roads (right). Blue-edged hollow boxes are the detection boxes and solid boxes are the tracked vehicles colored according to individual track IDs. Images show associated scenarios.}
\label{fig:sg}
\end{figure}
\end{comment}

\begin{figure*}[t!]
\centering
  %{\includegraphics[width=\linewidth,height=0.6\linewidth]{image/tracking_kitti_2.png}}
  %\hfill
  {\includegraphics[width=0.95\linewidth,height=0.35\linewidth]{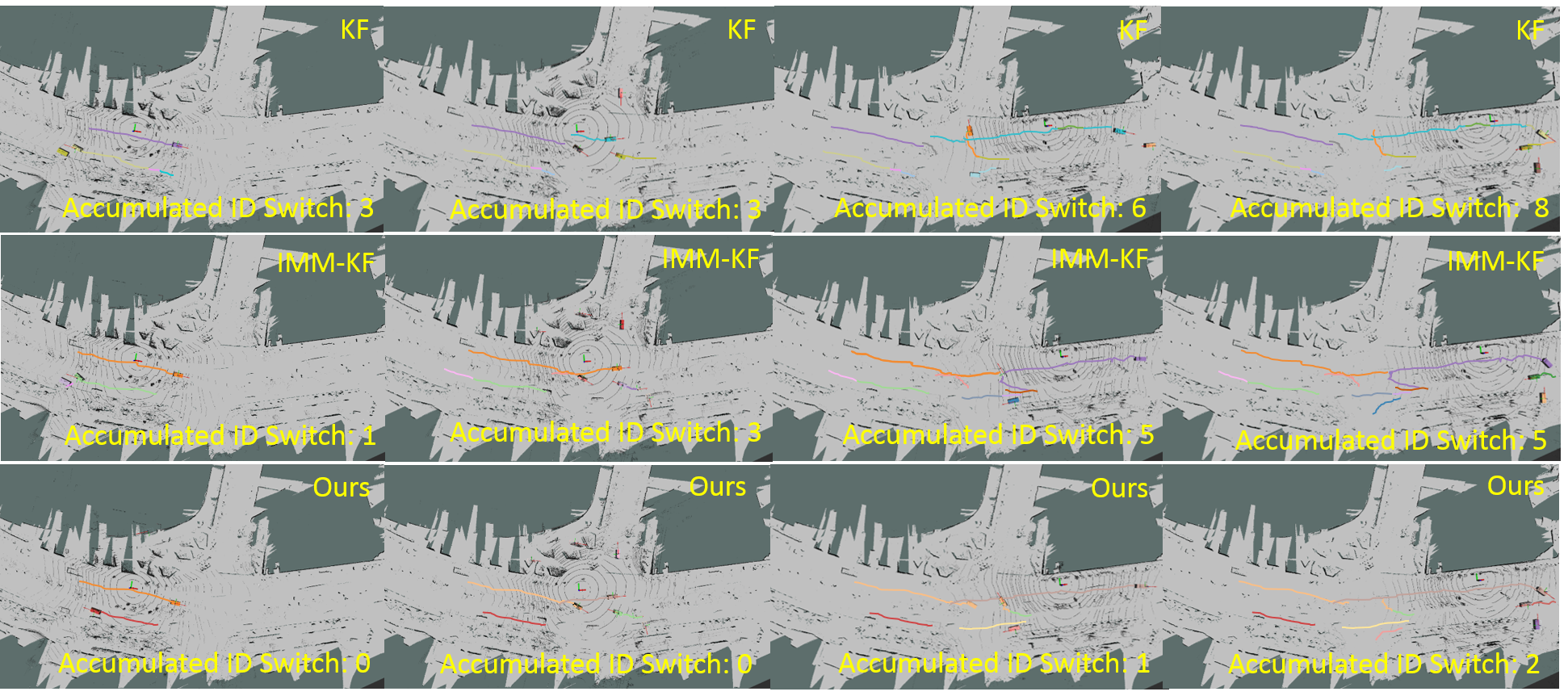}}
\caption{The comparison of tracking result in Singapore roads using different approaches. Bounding boxes and trajectories color coded by individual target IDs.}
\vspace{-1em}
\label{fig:compare}
\end{figure*}

In this section, we present the experimental results and analysis of multi-vehicle tracking using our proposed framework with Lidar only. We conduct two sets of tests to verify quantitatively and qualitatively the performance of road context -agnostic and -aware setups, respectively.  

\subsection{Road-context-agnostic Tracking on KITTI Dataset}
The first test is performed on the KITTI tracking dataset \cite{KITTI}. Instead of benchmarking with the best-ranking tracking-by-detection algorithms, the purpose/focus is to leverage the labelled groundtruths to validate the the effect of our proposed IMM posterior residual based association on MOT tracking, without using the road context. All the 21 episodes in the training set are used for evaluating the tracking performance on the \textit{Car} category and the metrics are calculated with 3D IoU evaluation proposed in \cite{Weng2019_3dmot}. Our 3D detector processes the $360^o$ point cloud of each frame and produces the bounding boxes for the cars detected, which are tracked later on using the KF-based MOT trackers. An instance is shown in \textbf{Fig.~\ref{fig:overview}}, top right.

The plain version of the classical KF-based tracking pipeline and the IMM version of it are selected as the baseline methods to compare with our approach, which use IoUs of the detection $\hat{z_t}$ and the prediction prior $H\hat{x_t}^{\mu_{t-1}}$ for data association. For the IMM, we employ 5 primitive models for the KFs, namely, constant velocity, constant acceleration,  coordinated turning, constant turning and velocity, constant turning and acceleration, with linearization performed whenever necessary. We initialize the mode probability uniformly over all models (i.e. $\mu_0^i=0.2$), and set a time-invariant default TPM:
\small
\begin{flalign} 
   & \mu_0^{} = 
  \begin{bmatrix} 
   0.2 \\
   0.2 \\
   0.2 \\
   0.2 \\
   0.2 \\
  \end{bmatrix} 
   , 
   Pr =
  \begin{bmatrix} 
   0.85 & 0.05 & 0.05 & 0.05 & 0.0  \\
   0.1 & 0.85 & 0.0 & 0.0 & 0.05  \\
   0.05 & 0.05 & 0.8 & 0.05 & 0.05  \\
   0.05 & 0.0 & 0.05 & 0.8 & 0.1  \\
   0.0 & 0.05 & 0.05 & 0.1 & 0.8  \\
  \end{bmatrix} 
\end{flalign}
\normalsize

\begin{table}[t!]
\caption{Quantitative Tracking Evaluation on KITTI Training Set for Car Category}
\label{tab:1}
\centering
\tiny
\begin{tabular}{c | c | c | c | c | c | c | c | c }
{\bf Benchmark} & \bf {MOTA} & {\bf MOTP} & {\bf MT} & {\bf ML} & {\bf FP} & {\bf FN} & {\bf IDS} & {\bf FRAG}\\ \hline
KF & 30.11 \% & 49.47 \% & 9.27 \% & 38.79 \% & 3314 & 12513 & 452 & 1034\\
\hline
IMM-KF & 33.78 \% & 56.68 \% & 11.17 \% & 36.70 \% & \textbf{3216} & 12461 & 261 & 816\\
\hline
Ours & \textbf{33.49} \% & \textbf{56.34} \% & \textbf{13.12} \% & \textbf{36.17} \% & 3240 & \textbf{12423} & \textbf{166} & \textbf{634}\\
\hline
\end{tabular}
\end{table}
\normalsize

The detector runs at 20 Hz on the video sequences provided by KITTI and the tracker is able to track at 100 Hz. The quantitative performance evaluation is shown in \textbf{Table~\ref{tab:1}}. The results show our method is able to obtain apparent improvement over the plain version and IMM version of the KFs, especially with significantly reduced ID-switch, demonstrating the robustness of our maneuver-orientated data association throughout the testing sequences. 

\begin{comment}
\begin{figure}[t!]
\centering
  %{\includegraphics[width=\linewidth,height=0.6\linewidth]{image/tracking_kitti_2.png}}
  %\hfill
  {\includegraphics[width=\linewidth,height=0.5\linewidth]{image/tracking_kitti_1.png}}
\caption{The detection and tracking instance on KITTI dataset. Blue-edged hollow boxes are the detection boxes and the tracked vehicles are shown by colored solid boxes according to individual target IDs.}
\label{fig:kitti}
\end{figure}
\end{comment}

\subsection{Road-context-aware Tracking in One-North Singapore}

The second test is performed with tracking frameworks implemented on our Autonomous Vehicle test bed driving in public roads of Singapore One-North region (\textbf{Fig.~\ref{fig:overview}}, bottom right). Like the road-context-agnostic test, we initialize the mode probability uniformly over all models and set a default TPM for each newly initiated track, except that the TPMs become time-varying in our proposed framework. The ego-car uses the map to localize itself as well as the detected cars in the map frame and dynamically adjust the TPMs for each target according to its neighbouring road context \footnote{Our implementation edits the traffic signals offline with \textit{rosbag} replay.}. 

To verify the effectiveness of the road text regulated tracking of maneuvering vehicles, we also demonstrate the tracking result across looped road segments at One-North and conduct a qualitative comparison. In \textbf{Fig.~\ref{fig:switch}}, the tracking history of a selected series of vehicles tracked using our proposed framework is shown, including the activated context vectors (magenta) by all tracked vehicles, and an exemplary illustration of the mode shift along the trajectory of selected vehicles. The trajectory is smoothly interpolated from the raw tracking result while the mode shift is calculated at the tracklet nodes, displayed as pie charts of the mode weights. To compare the tracking stability, the trajectories of tracked vehilces by the three methods and corresponding ID switch are elaborated. As shown in \textbf{Fig.~\ref{fig:compare}}, at road junctions, the plain version of the KF pipeline frequently loses the targets performing abrupt turning or stop/go maneuvers. In contrast, our approach with road context regulation tracks the vehicles continuously with relatively constant target IDs, even when the localization drifts. While the IMM version without context still tracks all vehicles at junctions and switching traffic lights, but the case of ID switch still happens more often than road context aided tracking. This is due to the ambiguous matches brought by maneuver-oriented data association giving rise to the probability of cross-model matching in IMM, while the fixed model data association using plain KF excludes possible maneuvers from the matching. The road context employed in our method balances the trade-off between these two setups.

The test sequence demonstrates our framework's ability of tracking frequently maneuvering vehicles while aware of its motion model shift, complying to dynamic road context.\footnote{More demonstrations available in the video link\\
\url{ https://drive.google.com/drive/folders/1xEj6e5QV3v8IYVA56nL-R3t0j_xtujEU?usp=sharing}}

\begin{figure}[t!]
\centering
  %{\includegraphics[width=\linewidth,height=0.6\linewidth]{image/tracking_kitti_2.png}}
  %\hfill
  {\includegraphics[width=\linewidth,height=0.65\linewidth]{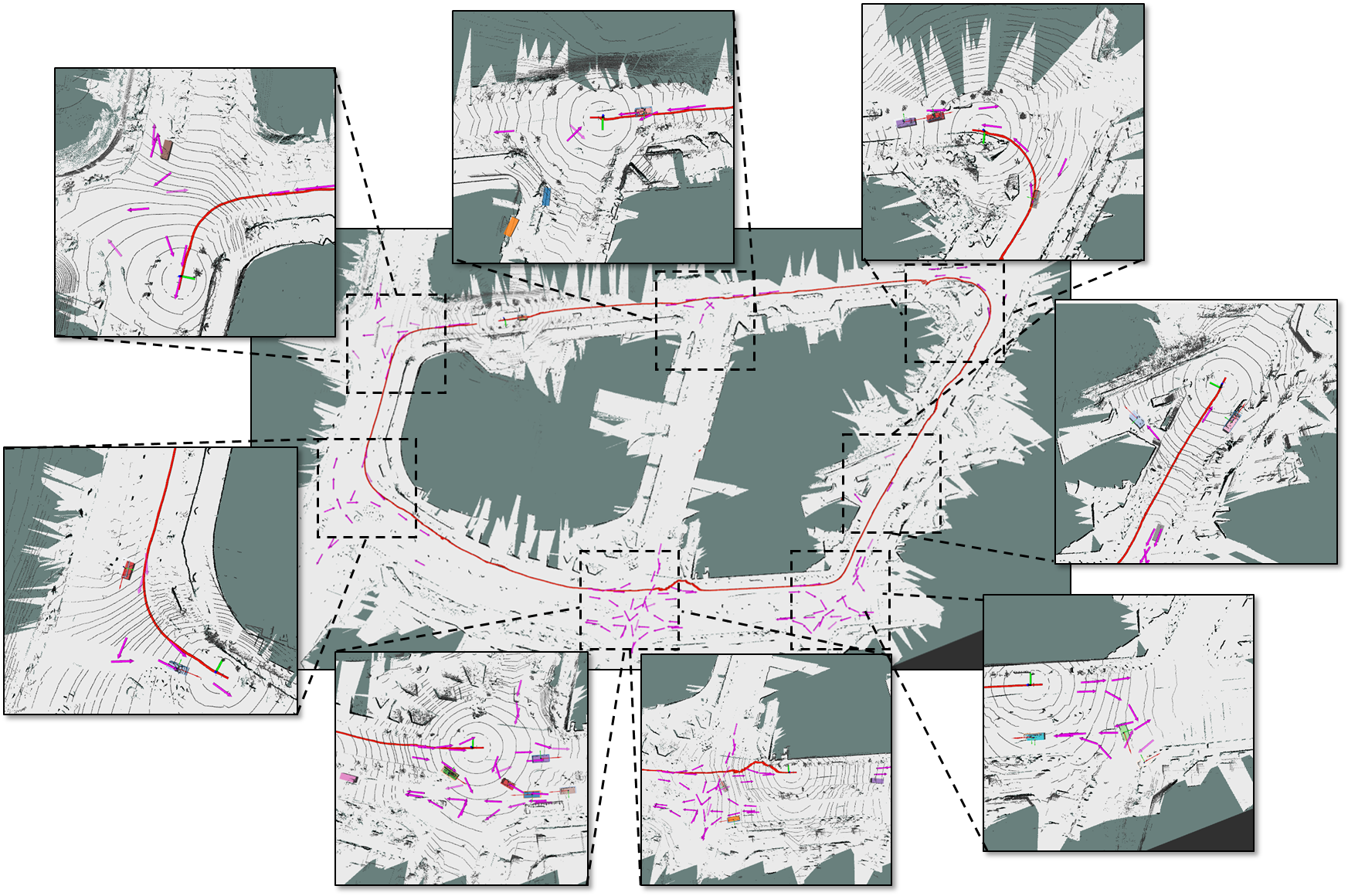}}
%\caption{The tracking history of a selected target vehicle.}
%\label{fig:switch}
  {\includegraphics[width=\linewidth,height=0.5\linewidth]{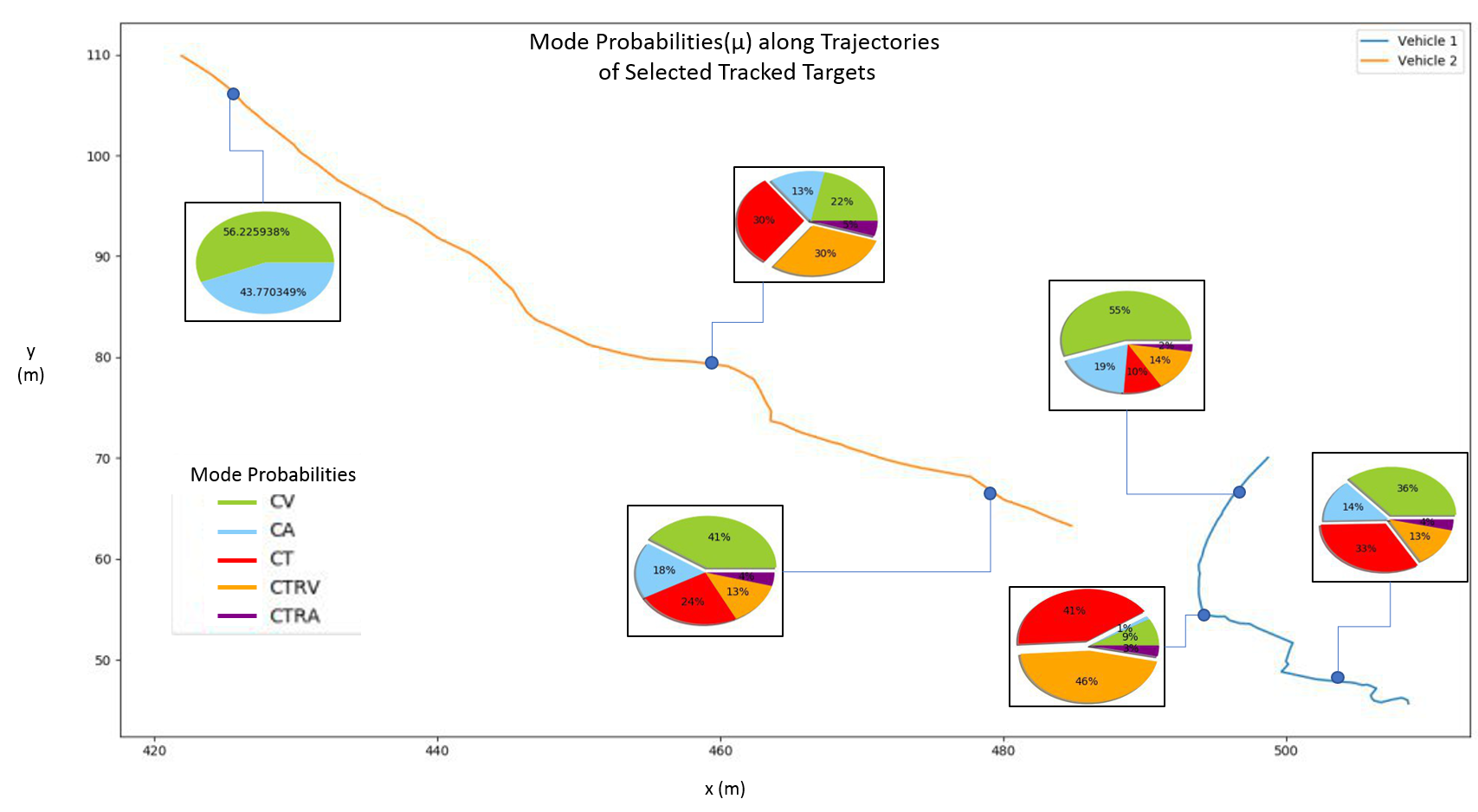}}
\caption{The tracking history of a selected target vehicle.}
\label{fig:switch}
\end{figure}

\section{Conclusion}
In this paper, we propose an online multi-target tracking-by-detection framework for maneuvering vehicle tracking under motion uncertainty. The framework consists of a 3D vehicle detector, a Kalman filter based MOT tracking backbone, and a maneuver-orientated data association using the Interacting Multiple Model (IMM) to account for the maneuvering uncertainty. The point cloud based detector provides real-time 3D bounding boxes of detected vehicle. The data association leverages the model-switching function of IMM to calculate the \textit{a-posterior} residual of each association hypothesis. Unified spatial-temporal association is achieved via integrating the road context into the IMM to adjust the time varying transition probability matrix (TPM). Deterministic online bipartite linear assignment is performed to match detections to existing tracks or create new tracks which are maintained by the Kalman filter with a lifespan management unit. Experimental results demonstrate the effectiveness of our framework tracking multiple vehicles along with dynamic road contexts, including changing road segments and traffic signs/signals. 

\newpage
\bibliographystyle{IEEEtran}
\bibliography{reference}

\addtolength{\textheight}{-12cm}   % This command serves to balance the column lengths

%%%%%%%%%%%%%%%%%%%%%%%%%%%%%%%%%%%%%%%%%%%%%%%%%%%%%%%%%%%%%%%%%%%%%%%%%%%%%%%%

%%%%%%%%%%%%%%%%%%%%%%%%%%%%%%%%%%%%%%%%%%%%%%%%%%%%%%%%%%%%%%%%%%%%%%%%%%%%%%%%

%%%%%%%%%%%%%%%%%%%%%%%%%%%%%%%%%%%%%%%%%%%%%%%%%%%%%%%%%%%%%%%%%%%%%%%%%%%%%%%%
\begin{comment}
\section*{APPENDIX}

Appendixes should appear before the acknowledgment.

\section*{ACKNOWLEDGMENT}

The preferred spelling of the word ÒacknowledgmentÓ in America is without an ÒeÓ after the ÒgÓ. Avoid the stilted expression, ÒOne of us (R. B. G.) thanks . . .Ó  Instead, try ÒR. B. G. thanksÓ. Put sponsor acknowledgments in the unnumbered footnote on the first page.

\end{comment}

%%%%%%%%%%%%%%%%%%%%%%%%%%%%%%%%%%%%%%%%%%%%%%%%%%%%%%%%%%%%%%%%%%%%%%%%%%%%%%%%

\end{document}